\documentclass[runningheads]{template/llncs}
\usepackage[T1]{fontenc}
\usepackage{graphicx}

\usepackage{pgfplots}
\pgfplotsset{compat=1.17}
\usepackage{tikz}
\usepackage{tikzscale}
\tikzstyle{every node}=[draw, fill=white, shape=circle, inner sep = 1.5pt]
\usepackage[draft]{packages/duerrmacros}
\usepackage{packages/fca}
\usepackage{centernot} %
\usepackage{csquotes}
\MakeOuterQuote{"}

\usepackage{color}

\begin{document}
\title{Maximal Ordinal Two-Factorizations}

\author{Dominik Dürrschnabel\inst{1,2}\orcidID{0000-0002-0855-4185} \and
    Gerd Stumme\inst{1,2}\orcidID{0000-0002-0570-7908}}
\authorrunning{D. Dürrschnabel and G. Stumme}

\institute{Knowledge \& Data Engineering Group,
    University of Kassel, Kassel,  Germany\\
    \and
    Interdisciplinary Research Center for Information System Design, Kassel, Germany\\
    \email{\{duerrschnabel, stumme\}@cs.uni-kassel.de} }

\maketitle              %
\begin{abstract}
    Given a formal context, an ordinal factor is a subset of its incidence relation that forms a chain in the concept lattice, i.e., a part of the dataset that corresponds to a linear order.
    To visualize the data in a formal context, Ganter and Glodeanu proposed a biplot based on two ordinal factors.
    For the biplot to be useful, it is important that these factors comprise as much data points as possible, i.e., that they cover a large part of the incidence relation.
    In this work, we investigate such ordinal two-factorizations.
    First, we investigate for formal contexts that omit ordinal two-factorizations the disjointness of the two factors.
    Then, we show that deciding on the existence of two-factorizations of a given size is an $\NP$-complete problem which makes computing maximal factorizations computationally expensive.
    Finally, we provide the algorithm \textsc{Ord2Factor} that allows us to compute large ordinal two-factorizations.
\end{abstract}

\section{Introduction}

A common way to analyze datasets with binary attributes is to treat them as numerical, i.e., the value 1 is assigned to each attribute-object incidence and 0 to each missing one.
Then, dimension reduction methods such as principal component analysis can be applied.
Thereby, multiple attributes are merged into a few number of axes while being weighted differently.
The emerging axes thus condense the presence or absence of correlated features.
The objects are embedded into these axes as follows.
Each object is assigned to a real-valued number in each axis to represent the best position of the object in the respective axis.
Thereby, the resulting placement of an object that only has some of these attributes yields an ambiguous representation.
Therefore, the main issues of this method arise.
Assigning a real value to an object is not consistent with the level of measurement of the underlying binary data.
It promotes the perception that an element has, compared to others, a stronger bond to some attributes, which is not possible in a formal context.
A method that encourages such comparisons and results in such an inaccurate representation of the original information is, in our opinion, not valid.
An example of such a principal component analysis projection is given in \cref{fig:pca}.
\begin{figure*}
    \centering
    \scalebox{0.8}{
        \includegraphics[height=20em]{tikz/Forum-Romanum-PCA.tikz}}
    \caption{A 2-dimensional projection of the objects from the dataset depicted in \cref{fig:dataset} using principal component analysis. This figure encourages the perception, that "Temple of Castor and Pollux" has a stronger bond to "GB1" than "Temple of Antonius and Fausta", which is false.}
    \label{fig:pca}
\end{figure*}

To address this problem, Ganter and Glodeanu~\cite{Ganter.2012} developed the method of \emph{ordinal factor analysis}.
It allows for a similar visualization technique while avoiding the problems that stem from the real-valued measurement on the binary attributes.
Once again, multiple attributes are merged into a single factor.
The computed projection in ordinal factor analysis thereby consists of linear orders of attributes, the so-called \emph{ordinal factors}.
Then, the method assigns each object, based on its attributes, a position in every factor.
Compared to the principal component analysis approach, the positions assigned in the process are natural numbers instead of real-valued ones.
Therefore, if interpreted correctly, the resulting projection does not express inaccurate and incorrect information.
A desirable property of ordinal factorizations is completeness which allows the deduction of all original information.
The positions of the objects are determined as follows.
Each object in the two-dimensional coordinate system is placed at the last position of each axis such that it has all attributes until this position.
Such a plot can be seen in \cref{fig:ordinal}.
Reading the biplot is done as follows:
Consider the "Portico of Twelve Gods" object. In the vertical factor, it is at position "GB1" which is preceeded by the attributes "M1", and "P" in this factor.
Thus, the object has all three of these attributes.
In the horizontal factor, it is at position "M1" which has no preceeding attributes.
These incidences together precisely represents the incidences of the object.
Deducing the same information from \cref{fig:pca} is hardly possible.

However, Ganter and Glodeanu do not provide a method for the computation of ordinal two-factorizations.
This is the point, where we step in with this work.
We provide an algorithm to compute ordinal two-factorizations if they exist.
We couple this algorithm with a method to compute a subset of the incidence relation of large size such that it admits an ordinal two-factorization.
This enables the computation of ordinal two-factorizations of arbitrary datasets.
This process combined results in the algorithm \textsc{Ord2Factor}.
Furthermore, we investigate ordinal two-factorizations with respect to their disjointness.

\begin{figure}[ht]
    \centering
    \scalebox{0.95}{
        \includegraphics[height=20em]{tikz/living-beings-and-water.tikz}}
    \caption{A biplot of a maximal ordinal two-factorization of the data from \cref{fig:dataset}. It represents all incidences from the formal context except (Temple of Romulus, GB1) and  (Basilica of Maxentius, B).}
    \label{fig:ordinal}
\end{figure}

\section{Related Work}

In this work, we consider a method that represent the objects in a low number of dimensions which condense multiple merged attributes.
A commonly applied approach with the same fundamental idea is principal component analysis~\cite{Pearson.1901}, which is a method that minimizes the average squared distances from the data points to a line.
An example for a principal component analysis projection is depicted in \cref{fig:dataset}.
For an extensive survey on dimensional reduction methods, we refer the reader to Espadoto et al.~\cite{Espadoto.2021}.
A comparison of principal component analysis with the methods from formal concept analysis was described by Spangenberg and Wolff~\cite{Spangenberg.1991}.

Our work is located in the research area formal concept analysis~\cite{fca-book} (FCA).
One well-researched way to apply dimensionality reduction in FCA is Boolean factor analysis ~\cite{Keprt.2004,Keprt.2006,Belohlavek.2007,Belohlavek.2010,Boeck.1988}.
Thereby, the incidence relation is represented by families of incident attribute and object subsets.
This idea is highly related to the notion of a formal concept in FCA.
Based on this method, Ganter and Glodeanu propose in \cite{Ganter.2012} to group multiple Boolean factors into many-valued factors.
One research direction that they consider to be of special interest is grouping them into ordinal factors, which is the direction we follow in this work.
The same authors also demonstrate~\cite{Glodeanu.2013} that the research method is useful in application.
The theory was also lifted to the triadic case~\cite{Glodeanu.2013_3} and graded data \cite{Glodeanu.2014}.
In our previous work~\cite{greedy_duerr}, we propose an algorithm to greedily compute ordinal factors in large formal contexts.
In a broader sense, the discovery of substructures, such as induced contranominal scales~\cite{contranominals} or ordinal motifs~\cite{Hirth.2023} is an often-considered problem in the context of structural investigations in formal concept analysis.

The methods developed in this paper can be considered dual to the \textsc{DimDraw}-algorithm~\cite{dimdraw}.
In both approaches, two linear orders are sought-after to represent the dataset which is done using a connection to a bipartite subgraph.

\begin{figure}[t]
    \centering
    \begin{cxt}
        \att{B}
        \att{GB1}
        \att{GB2}
        \att{M1}
        \att{M2}
        \att{M3}
        \att{P}
        \obj{XX.XX.X}{Arch of Septimus Severus}
        \obj{XXXXX..}{Arch of Titus}
        \obj{...X...}{Basilica Julia}
        \obj{X......}{Basilica of Maxentius}
        \obj{......X}{Curia}
        \obj{...X...}{House of the Vestals}
        \obj{.X.XX..}{Phocas column}
        \obj{.X.X..X}{Portico of Twelve Gods}
        \obj{XX.XXXX}{Temple of Antonius and Fausta}
        \obj{XXXXXXX}{Temple of Castor and Pollux}
        \obj{.X.....}{Temple of Romulus}
        \obj{...XX.X}{Temple of Saturn}
        \obj{...XX..}{Temple of Vespasian}
        \obj{.XXXX.X}{Temple of Vesta}
    \end{cxt}
    \caption{Running example: This dataset compares attributes of different social media platforms.}
    \label{fig:dataset}
\end{figure}

\section{Foundations of Ordinal Factor Analysis}
In this section we briefly recap the definitions and notions from graph theory, formal concept analysis and ordinal factor analysis that are necessary to understand this paper.
A graph is a tuple $(V,E)$ with $E \subset \binom{V}{2}$. The set $V$ is called the set of \emph{vertices} and $E$ the set of \emph{edges}. A path between two vertices $v_1$ and $v_n$ is a sequence of vertices $v_1\ldots v_n$ with $\{v_i, v_{i+1}\}\in E$. A (connected) component is a maximal subset of the vertices of a graph, such that between every pair of vertices in the component there is a path.
An \emph{ordered set} is a tuple $(B, \leq)$ where $\leq$ is a binary relation on $B \times B$ that is reflexive, antisymmetric and transitive.
The cocomparability graph of the ordered set $(B,\leq)$ is the graph $(B,E)$ where $\{a,b\} \in E$ if $a \nleq b$ and $b \nleq a$.
A \emph{formal context} is a triple $(G,M,I)$, where $G$ is called the set of \emph{objects}, $M$ is called the set of \emph{attributes} and $I\subset G \times M$ is called the incidence relation.
The two \emph{derivation operators} between the power sets of attributes and objects are given by $A'=\{m \in M \mid \forall g \in A : (g,m) \in I\}$ and $B'=\{g \in G \mid \forall m \in B: (g,m) \in I\}$.
A formal concept is a tuple $(A,B)$ with $A \subset G$ and $B \subset M$ such that $A' = B$ and $B' = A$.
The set of all formal concepts with is denoted by $\B$ and forms together with the order relation $\leq$ for which $(A_1,B_1) \leq (A_2, B_2)$ if and only if $A_1 \subset A_2$, the concept lattice $\BV=(\B, {\leq})$.

From here on we introduce the notions that are important for the present work.
A \emph{Ferrers relation} is a binary relation $F \subset G \times M$ with the property that if $(g,m) \in F$ and $(h,n) \in F$, either $(g,n) \in F$ or $(h,m) \in F$.
An \emph{ordinal two-factorization} of a formal context $(G,M,I)$ is a set of two Ferrers relations $F_1$ and $F_2$ such that $F_1 \cup F_2 = I$.
The \emph{incompatibility graph} of a formal context $(G,M,I)$ is the graph $(I,E)$ with $\{(g,m),(h,n)\} \in E$ if and only if $(g,n) \not \in I$ and $(h,m) \not \in I$.
We also say that $(g,n)$ and $(h,m)$ are \emph{incompatible}.
It is known, that a formal context admits an ordinal two-factorization, if and only if the incompatibility graph is bipartite~\cite[Sec.2,\,Prop.2]{Doignon.1984}.
Also note, that it is possible to extend an ordinal factor, i.e., a Ferrers relation, to a chain of formal concepts~\cite[Prop.7]{Ganter.2012}.

\section{Disjointness of Ordinal Two-Factorizations}

It is of interest wether the computed two ordinal factors are disjoint, i.e., if there is no information that is represented by both factors.
Otherwise, the ordinal factorization contains redundant information as the two Ferrers relations share pairs of the incidence relation.
However, as the formal context in \cref{cxt:factor_counterexample} shows, achieving two disjoint ordinal factors is not always possible.
The example is due to Das et al.~\cite{Das.1989}, where they characterize the formal contexts that are factorizable into two disjoint ordinal factors as interval digraphs.
The incompatibility graph of the example consists of two connected components, one consists of the single vertex $(6,f)$ and the other the vertices $I \setminus \{(6,f)\}$.
We can assign the bipartition classes of this second component to factor 1 and 2 without loss of generality as shown on the right side of the \cref{cxt:factor_counterexample}.
But then the incidence pair then $(6,f)$ has to be in both factors, in factor 1 because as $(6,e)\in I$ and $(2,f) \in I$ but $(2,e)\notin I$ and in factor 2 because of $(5,f)\in I$, $(6,b)\in I$, and $(5,b)\notin I$.

\begin{figure}[b]
    \centering
    \null\hfill%
    \begin{cxt}
        \att{a}%
        \att{b}%
        \att{c}%
        \att{d}%
        \att{e}%
        \att{f}%
        \att{g}%
        \obj{...XXXX}{1}%
        \obj{.....XX}{2}%
        \obj{......X}{3}%
        \obj{X......}{4}%
        \obj{X....X.}{5}%
        \obj{XX..XXX}{6}%
        \obj{XXX..X.}{7}%
    \end{cxt}%
    \hfill
    \begin{cxt}
        \att{a}%
        \att{b}%
        \att{c}%
        \att{d}%
        \att{e}%
        \att{f}%
        \att{g}%
        \obj{...1111}{1}
        \obj{.....11}{2}
        \obj{......1}{3}
        \obj{2......}{4}
        \obj{2....2.}{5}
        \obj{22..1X1}{6}
        \obj{222..2.}{7}
    \end{cxt}%
    \hfill\null%
    \caption{Example of a context with a maximal bipartite subgraph that does not give rise to an ordinal two-factorization. This example is due to Das et al.~\cite{Das.1989}.}
    \label{cxt:factor_counterexample}
\end{figure}

We know that a formal context admits a two-factorization if its incompatibility graph is bipartite, and we can deduce from its incompatibility graph incidence pairs that cannot appear in the same ordinal factor.
Still, we note that even in cases where disjoint two-factorizations do exist, the bipartition classes of the incompatibility graph do not necessarily give rise to an ordinal two-factorization.
To see this, refer to \cref{fig:ord_counterexample}.
The incompatibility graph (middle) of the formal context (left) consists of three components.
An assignment of the incompatibility graph to bipartition classes can be seen on the right, however the elements $(2,a)$ and $(3,b)$ of factor 2 would imply, that the element $(3,a)$ also has to be in factor 2.
However, this element is incompatible to element $(3,c)$, which is also in factor 2 by the assignment of the incompatibility graph.
Thus, such an assignment does not always result in a valid ordinal two-factorization.

\begin{figure}[t]
    \hspace*{-1.8em}
    \null\hfill
    \raisebox{3.2em}{
        \begin{cxt}
            \att{a}
            \att{b}
            \att{c}
            \obj{.xx}{1}
            \obj{x.x}{2}
            \obj{xx.}{3}
        \end{cxt}}
    \hfill
    \begin{tikzpicture}
        \node[shape=circle,draw=black] (A) at (0,0) {3,a};
        \node[shape=circle,draw=black] (B) at (0,.75) {2,a};
        \node[shape=circle,draw=black] (C) at (.75,0) {3,b};
        \node[shape=circle,draw=black] (D) at (.75,1.5) {1,b};
        \node[shape=circle,draw=black] (E) at (1.5,.75) {2,c};
        \node[shape=circle,draw=black] (F) at (1.5,1.5) {1,c} ;

        \path [-](A) edge  (F);
        \path [-](B) edge  (D);
        \path [-](C) edge  (E);
    \end{tikzpicture}
    \hfill
    \raisebox{3.2em}{\begin{cxt}
            \att{a}
            \att{b}
            \att{c}
            \obj{.12}{1}
            \obj{2.1}{2}
            \obj{12.}{3}
        \end{cxt}}
    \hfill\null
    \caption{\emph{Left:} The formal context of a contranominal scale. \emph{Middle:} its   comparability graph. \emph{Right:} A bipartition of the transitive comparability graph that does not give rise to an ordinal two-factorization.}
    \label{fig:ord_counterexample}
\end{figure}

On the other hand, if the incompatibility graph is connected and bipartite, any assignment of the elements to bipartition classes of the incompatibility graph generates a valid ordinal two-fac\-to\-ri\-za\-tion, as the following shows.

\begin{proposition}
    Let $\K$ be a formal context with a connected and bipartite incompatibility graph.
    Then there are two unique disjoint factors $F_1$ and $F_2$ that factorize $\K$. The sets
    $F_1$ and $F_2$ correspond to the bipartition classes of the incompatibility graph.
\end{proposition}

\begin{proof}
    As the cocomparability graph is bipartite, $\K$ admits an ordinal two-fac\-tor\-iza\-tion.
    As the incompatibility graph is connected, it has unique bipartition classes.
    Finally, two elements in the same bipartition class cannot appear in the same ordinal factor.
\end{proof}

On the other hand, if the incompatibility graph is not connected, it is of interest to further investigate the incidence pairs that can appear in both ordinal factors.
We do so in the following.

\begin{lemma}
    Let $\K=(G,M,I)$ be a formal context with bipartite incompatibility graph $(I,E)$.
    Let $F_1,F_2$ be an ordinal two-factorization of $\K$.
    For all elements $(g,m)\in F_1 \cap F_2$ it holds that $\{(g,m)\}$ is a connected component in $(I,E)$.
\end{lemma}

\begin{proof}
    Assume not, i.e., there is an element $(g,m) \in F_1 \cap F_2$ that is not its own component.
    Then, there has to be some element $(h,n) \in I$ that is incompatible to $(g,m)$, i.e., $(g,n) \notin I$ and $(h,m) \notin I$.
    But then $(h,n)$ can be in neither $F_1$ nor $F_2$ which contradicts the definition of a Ferrers relation.
\end{proof}

Thus, only the elements of the incompatibility graph that are not connected to another element can be in both ordinal factors.
In the following, we further characterize these isolated elements, as we show that for these elements it is always possible that they are in both factors, which then fully characterizes the intersection of the two ordinal factors.

\begin{lemma}
    Let $\K=(G,M,I)$ be a formal context with a two-factorization $F_1$, $F_2$.
    Let $C$ be the set of all elements of $I$ that are incompatible to no other element.
    Then $F_1\cup C$, $F_2\cup C$ is also an ordinal two-factorization.
\end{lemma}

\begin{proof}
    Let $G_i = F_i \setminus C$ and $\tilde{F}_i= F_i\cup C$ for $i \in \{1,2\}$
    Assume the statement is not true, i.e., either $\tilde{F}_1$ or $\tilde{F}_2$ is no ordinal factor.
    Without loss of generality, let $\tilde{F}_1= F_1\cup C$ be no ordinal factor.
    Then, there are two elements $(g,m),(h,n) \in \tilde{F}_1$ such that $(g,n)\notin \tilde{F}_1$ and $(h,m)\notin \tilde{F}_1$.
    As $F_1$ is an ordinal factor, at least one of the two elements has to be in $C$, let without loss of generality $(g,m) \in C$.
    We now do a case distinction whether one or both of them are in $C$.\\
    \emph{Case 1.} Let first $(g,m) \in C$ and $(h,n) \notin C$.
    One of the elements $(g,n)$ or $(h,m)$ has to be in $I$, otherwise $(g,m)$ and $(h,n)$ are incompatible, without loss of generality, let $(h,m) \in I$.
    As $(h,m) \notin C$, it has to hold that $(h,m) \in G_2$ and there has to be some $(x,y) \in G_1$ with $(h,y) \notin I$ and $(x,m) \notin I$.
    As $(x,y) \in G_1$, $(h,n) \in G_1$, and $(h,y) \notin I$ and $F_1$ is an ordinal factor, $(x,n) \in F_1$.
    As $(g,m)\in C$ it is incompatible with no element and thus not incompatible with $(x,n)$ in particular, but $(x,m) \notin I$, it holds that $(g,n)\in I$.
    The element $(g,n)$ has to be in $G_2$, as otherwise $(g,m)$ and $(h,n)$ are not incompatible.
    Thus, there also has to be some element in $(a,b) \in G_1$ with $(a,n) \notin I$ and $(g,b) \notin I$.
    As $F_1$ is an ordinal factor and $(x,n) \in F_1, (a,b) \in F_1$ and $(a,n) \notin F_1$, the element $(x,b)\in F_1$.
    But then $(x,b)$ is incompatible to $(g,m)$ which is a contradiction to $(g,m)$ being in $C$.\\
    \emph{Case 2.}
    Let $(g,m) \in C$ and $(h,n) \in C$. Either $(g,n) \in I$ or $(h,m) \in I$, let without loss of generality $(g,n) \in I$.
    Then it has to be hold more specifically that in $(g,n) \in G_2$.
    Thus, there is some element $(x,y) \in G_1$ with $(h,y) \notin I$ and $(x,m) \notin I$.
    Because $(x,y)$ has to be compatible with $(h,n)$, it has to hold that $(x,n) \in I$.
    On the other hand, $(x,n)$ has to be compatible with $(g,m)$, thus $(g,n) \in I$.
    It then has to hold that $(g,n) \in G_2$ and thus some element $(a,b) \in G_1$ has to exist with $(g,b) \notin I$ and $(a,n) \notin I$.
    As $(x,y) \in F_1$ and $(a,b) \in F_1$ and $F_1$ is an ordinal factor, either $(a,y) \in F_1$ or $(x,b) \in F_1$.
    If $(a,y) \in F_1$, it would be incompatible to $(g,m)$, if $(x,b) \in F_1$ it would be incompatible to $(h,n)$. Both would be a contradiction to the respective element being in $C$.

\end{proof}

This theorem finishes a complete characterization of the non-disjoint part of ordinal factors.
Therefore, a partition of the incidence as follows always exists, if the context is ordinal two-factorizable.

\begin{theorem}
    Let $\K=(G,M,I)$ be a two-factorizable formal context with $(I,E)$ its incompatibility graph.
    Then there is a partition of $I$ into $F_1, F_2,C$ with $C = \{D \mid D $ connected component of $(I,E), |D| = 1\}$ and $F_1\cup C$ and $F_2 \cup C$ are Ferrers relations.
\end{theorem}

\begin{proof}
    This theorem directly follows from the previous lemma.
    Let $\tilde{F}_1, \tilde{F}_2$ be a two-factorization and $C = \{D \mid D $ connected component of $(I,E), |D| = 1\}$.
    Then the partition is given by $\tilde{F}_1 \setminus C$, $\tilde{F}_2 \setminus C$, and $C$.
\end{proof}

\section{An Algorithm for Ordinal Two-Factorizations}

Now, we propose an algorithm to compute ordinal two-factorizations if they exist.
As we saw in the last section, the bipartition classes of the incompatibility graph do not directly give rise to an ordinal two-factorizations.
An important observation~\cite[Thm.46]{fca-book} is that a formal context can be described by the intersection of two Ferrers relations, if and only if its corresponding concept lattice can be described as the intersection of two linear orders, i.e., if it has order dimension two.
As the complement of a Ferrers relation is once again a Ferrers relation, a formal context is two-factorizable if and only if the concept lattice of its complement context has order dimension two.
We leverage this relationship with the following theorem, to explicitly compute the ordinal two-factorization.

\begin{theorem}
    Let $\K = (G,M,I)$ be a formal context.
    Let $(B, \leq)=\BV(\compl{\K})$.
    If $\K$ is two-factorizable, then $\leq$ is two-dimensional and there is a conjugate order $\leq_c$.
    The sets
    \begin{align*}
        F_1=\{(g,m) \in G \times M \mid \nexists (A,B), (C,D) \in \B(\compl\K), g \in A, m \in D, & \\((A,B),(C,D)) \in ({\leq} \cup {\leq_c})&\}
    \end{align*}
    and
    \begin{align*}
        F_2=\{(g,m) \in G \times M \mid \nexists (A,B), (C,D) \in \B(\compl\K), g \in A, m \in D, & \\((A,B),(C,D)) \in ({\leq} \cup {\geq_c})&\}
    \end{align*}
    give rise to an ordinal factorization of $\K$.
\end{theorem}

\begin{proof}
    We have to show that $F_1$ and $F_2$ are Ferrers relations and $F_1 \cup F_2 = I$.
    We first show that $F_1$ is a Ferrers relation.
    By definition $\lessdot \coloneqq {\leq} \cup {\leq_c}$ is a chain ordering all formal concepts of $\B(\compl{K})$.
    Let $(g,m)$ and $(h,n)$ be two pairs in $F_1$.
    Assume that $(g,n) \notin F_1$ and $(h,m) \notin F_1$.
    Then there have to be two concept $(A_1,B_1)$ and $(A_2,B_2)$ with $g \in A_1$, $n \in B_2$ such that $(A_1,B_1) \lessdot (A_2,B_2)$.
    Similarly, there have to be two concepts $(A_3,B_3)$ and $(A_4,B_4)$ with $h \in A_3$, $m \in B_4$ such that $(A_3,B_3) \lessdot (A_4,B_4)$.
    For the concepts $(A_2,B_2)$ and $(A_3,B_3)$ it holds that $(A_2,B_2) \lessdot (A_3,B_3)$, as $(A_2,B_2) \neq (A_3,B_3)$ and  $(A_3,B_3) \centernot\lessdot (A_2,B_2)$ because $(g,m) \notin F_1$.
    By the same argument, $(A_4,B_4) \lessdot (A_1,B_1)$.
    Thus, $(A_1,B_1)\lessdot (A_2,B_2) \lessdot (A_3,B_3) \lessdot (A_4,B_4) \lessdot (A_1,B_1)$, which would imply that these three concepts are equal and is thus a contradiction.
    This proves that $F_1$ is a Ferrers relation.
    The argument to shows that $F_2$ is a Ferrers relation is dual.
    We now show that $F_1 \cup F_2 = I$.
    Let $(g,m) \in F_1 \cup F_2$, without loss of generality let it be an element of $F_1$.
    Then there are no two concepts $(A,B), (C,D) \in \B(\compl\K)$, with $g \in A$, $m \in D$ and
    $((A,B), (C,D)) \in ({\leq} \cup {\leq_c})$.
    Consider the concept $(g'',g')$ with the derivation from the complement context.
    By definition $((g'',g'), (g'',g')) \in ({\leq} \cup {\leq_c})$ and thus $m \notin g'$ when using the derivation from the complement context, i.e., $(g,m)\notin \compl{I}$.
    But then, $(g,m) \in I$.
    Now, let $(g,m) \in I$ and assume that $(g,m)\notin F_1$.
    Let $(A,B)$ and $(C,D)$ be arbitrary concepts of $\B(\compl\K)$ with $g \in A$ and $m \in D$.
    As $(g,m) \in I$, it is not in the incidence of $\compl\K$ and thus $(A,B)$ and $(C,D)$ are not comparable with $\leq$.
    As $(g,m) \notin F_1$ it holds $(A,B) \nleq_c (C,D)$ and as they are incomparable with $\leq$ it has to hold that $(A,B) \geq_c (C,D)$.
    Thus, $(g,m)\in F_2$.
    This shows that $F_1 \cup F_2 = I$ and thus concludes the proof.
\end{proof}

This proof gives rise to the routine in \cref{alg:factor2d} where this information is used to compute an ordinal two-factorization of the formal context.
It computes the two sets $F_1$ and $F_2$ from the previous theorem.
To do so, it has to be paired with an algorithm to compute the concept lattice.
An algorithm that is suitable is due to Lindig~\cite{Lindig.2000}, as it computes the covering relation together with the concept lattice.
Furthermore, an algorithm for transitive orientations~\cite{Golumbic.1977} is required to compute the conjugate order.
As we will discuss later, for both these algorithms the runtime is not critical, as we are interested in ordinal two-factorizations of small formal contexts.
Modern computers are easily able to deal with such contexts.
Still, if suitable supporting algorithms are chosen, this procedure results in a polynomial-time algorithm that computes ordinal two-factorizations.

\begin{algorithm}[t]
    \caption{Compute Ordinal Two-Factorization}
    \label{alg:factor2d}
    \vspace*{0.3em}
    \textbf{Input:}   Ordinal Two-Factorizable Formal Context $\K=(G,M,I)$     \\
    \textbf{Output:}  Ordinal Two-Factorization $F_1$, $F_2$
    \vspace*{0.5em}
    \hrule
    \begin{lstlisting}
def two_factor$(G,M,I)$:
    $(\B, \leq)$ $=$ $\BV(\compl{\K})$
    $(\B,E)$ $=$ co_comparability_graph$(\B, \leq)$
    $\leq_c$ $=$ transitive_orientation$(\B,E)$
    $\lessdot_1$ = ${\leq} \cup {\leq_c}$
    $\lessdot_2$ = ${\leq} \cup {\geq_c}$
    for $i$ in $\{1,2\}$:
        $L_i$ $=$ $\{\}$
        $\tilde{A}$ $=$ $\{\}$
        for $(A,B)$ in $\B$ ordered by $\lessdot_i$:
            $\tilde{A}$ $=$ $\tilde{A} \cup A$
            $L_i$ $=$ $L_i \cup (\tilde{A} \times B)$
        $F_i$ $=$ $(G \times M) \setminus L_i$
    return $F_1, F_2$
  \end{lstlisting}
\end{algorithm}

\section{Maximal Ordinal Two-Factorizations}

In this section, we propose an algorithm to compute an ordinal two-factorizations of a formal context that covers a large part of the incidence relation.

\begin{definition}[Maximal Ordinal Two-Factorizations]
    For a formal context $\K=(G,M,I)$ a maximal ordinal two-factorization is a pair of Ferrers relations $F_1, F_2 \subset I$ such that there are no Ferrers relations $\tilde{F}_1, \tilde{F}_2 \subset I$ with $|\tilde{F}_1 \cup \tilde{F}_2| \geq |F_1 \cup F_2|$.
\end{definition}

While there are various thinkable ways, how to define maximal ordinal two-factorizations, we have chosen to do so by minimizing the size of not-covered incidence, as each element in the incidence can be viewed as a data point that would thus be lost.
This definition also aligns with Ganter's suggestion in his textbook to maximize the size of the union \cite{Ganter.2013}.
Note, that a maximal two-factorization always exists, as a single element in the incidences is a Ferrers relation by itself.

\subsection{Maximal Ordinal Two-Factorizations are Hard}

First, we investigate the computational complexity of the \textsc{Maximal Ordinal Two-Factorization Problem}.
To do so, we perform a reduction from the \textsc{Two-Dimension Extension Problem}.
The problem requires finding the minimum number of pairs that need to be added to an order relation to make it two-dimensional.
Formally, given an ordered set $(X, {\leq})$ and a $k \in \N$ requests the decision whether there is a set ${\tilde{\leq}} \supset {\leq}$ such that is an order and $(X,{\tilde{\leq}})$ has order dimension two and $|{\tilde{\leq}}|-|{\leq}|=k$.
Felsner and Reuter~\cite{Felsner.1999} showed that deciding this problem is $\NP$-complete.

The decision problem, that is linked to the minimum number of pairs that we need to add to a relation to make it two-dimensional is the \textsc{Ordinal Two-Factorization Problem}.
For it, a formal context $(G,M,I)$ and a $k \in \N$ are given.
It is requested to decide the existance of a formal context $(G,M,\tilde{I})$ with $\tilde{I} \subset I$ and $|I|-|\tilde{I}| = k$ that has an ordinal two-factorization.

The relation between these two problems gives rise to the computational complexity of computing a two-factorization.

\begin{lemma}
    \label{two_red}
    There is a polynomial-time reduction from the \textup{\textsc{Two-Dimension Extension Problem}} to the \textup{\textsc{Ordinal Two-Factorization Problem}}.
\end{lemma}

\begin{proof}
    Let $(X,\leq)$ and $k\in \N$ be an instance of the \textsc{Two-Dimension Extension Problem}. We claim, that the problem has a solution if and only if \textsc{Ordinal Two-Factorization Problem} $(X,X,{\nleq})$ with $k$ has a solution.

    Let $(X,{\leq})$ be an ordered set with a two-dimension-extension $C$ of size $k$.
    Let $L_1$, $L_2$ be a realizer, i.e., two linear extensions of ${\leq} \cup C$ with $L_1 \cap L_2 = {\leq} \cup C$.
    Then the relations $F_1 \coloneqq (X \times X)  \setminus L_1$ and $F_2 \coloneqq (X \times X)  \setminus L_2$ are Ferrers relations.
    Furthermore, $F_1 \subset {\nleq}$ and $F_2 \subset {\nleq}$ by definition.
    Assume that $F_1 \cup F_2 \cup C \neq {\nleq}$.
    Then there has to be a pair $a,b \in X$ with $a \nleq b$ and $(a,b)\notin C$.
    Then $(b,a) \in L_1$, or $(b,a) \in L_2$, or both, without loss of generality let $(b,a)\in L_1$.
    But this implies that $(a,b) \in L_1$ which is a contradiction.

    Now, let for the formal context $(X,X,\nleq)$ be $C$ a set of cardinality $k$ such that there are two Ferrers relations $F_1, F_2$ with $|F_1 \cup F_2 \cup C| = |{\nleq}|$.
    Now, let $L_1=(X \times X) \setminus F_1$ and $L_2=(X \times X) \setminus F_2$.
    Then it holds that $L_1 \cap L_2 = {\leq} \cup C$.
    $L_1$ and $L_2$ are supersets of $\leq$ and Ferrers relations, thus they are transitive, reflexive and for all elements $a,b \in X$ it holds that $a,b \in L_i$ or $b,a \in L_i$.
    By definition for both $i \in \{1,2\}$ there is a $\tilde{L}_i \subset L_i$ that has all these properties and is also antisymmetric.
    The existence of $\tilde{L}_i$ follows from placing a linear order on each of the equivalence classes of $L_i$.
    Both $\tilde{L}_1$ and $\tilde{L}_2$ are linear extensions of $\leq$ and $|\tilde{L}_1 \cap \tilde{L}_2| \leq |L_1 \cap L_2| = |{\leq}|+k$.

    Therefore, the claim is proven. and  we reduced the \textsc{Two-Dimension Extension Problem} to the \textsc{Ordinal Two-Factorization Problem}.

\end{proof}

\begin{lemma}
    \label{two_val}
    Validation of a solution of the \textup{\textsc{Ordinal Two-Factorization}} can be done in polynomial time.
\end{lemma}

\begin{proof}
    Given a formal context $(G,M,I)$ and a set $C\subset I$ of size $k$, to check whether $(G,M,I\setminus C)$ admits an ordinal two-factorization is equivalent to the check whether the  incompatibility graph of $(G,M,I\setminus C)$ is bipartite.
\end{proof}

Thus, the \textsc{Ordinal Two-Factorization Problem} is in the same complexity class as the \textsc{Two-Dimension Extension Problem}.

\begin{theorem}
    The \textup{\textsc{Ordinal Two-Factorization Problem}} is $\NP$-complete.
\end{theorem}

\begin{proof}
    Follows from \cref{two_red,two_val}.
\end{proof}

\subsection{Maximal Bipartite Subgraphs are Not Sufficient}
\label{fac-suff}

\begin{figure}
    \centering
    \begin{cxt}
        \att{a}%
        \att{b}%
        \att{c}%
        \att{d}%
        \att{e}%
        \att{f}%
        \att{g}%
        \att{h}%
        \att{i}%
        \att{j}%
        \att{k}%
        \att{l}%
        \att{m}%
        \att{n}%
        \att{o}%
        \att{p}%
        \att{q}%
        \att{r}%
        \obj{.XXX.X.X.XXXXXXX..}{1}
        \obj{X.XXXXXXXXXXXXX.X.}{2}
        \obj{XX.XXXXXXXXXXXX..X}{3}
        \obj{XXX.XX..XX.X.X.XXX}{4}
        \obj{XXXX.X.X.XXXXXXXXX}{5}
        \obj{XXXXX.X..X.X.X.XXX}{6}
        \obj{XXXXXX.XXXXXXXXXXX}{7}
        \obj{XXXXXXX.XX.X.X.XXX}{8}
        \obj{XXXXXXXX.XXXXXXXXX}{9}
        \obj{XXXXXXXXX.XX..X.XX}{10}
        \obj{XXXXXXXXXX.X.X.XXX}{11}
        \obj{XXXXXXXXXXX.X...XX}{12}
        \obj{XXXXXXXXXXXX.XXXXX}{13}
        \obj{XXXXXXXXXXXXX.X.XX}{14}
        \obj{XXXXXXXXXXXXXX.XXX}{15}
        \obj{XXXXXXXXXXXXXXX.XX}{16}
        \obj{XXXXXXXXXXXXXXXX.X}{17}
        \obj{XXXXXXXXXXXXXXXXX.}{18}
    \end{cxt}

    \caption{Example of a context with a maximal bipartite subgraph that does not give rise to an ordinal two-factorization.}
    \label{cxt:factor_counterexample2}
\end{figure}

The structure of the incompatibility graph provides an interesting foundation to compute ordinal two-factorizations.
It seems to be a tempting idea to compute the maximal induced bipartite subgraph of a formal context.
The vertex set that induces such a maximal bipartite subgraph could than be used for the ordinal two-factorization.
However, it turns out a bipartite subgraph does not always give rise to an ordinal two-factorization as \cref{cxt:factor_counterexample2} demonstrates.
Its incompatibility graph has an odd cycle, and it is thus not bipartite.
This fact makes the formal context not two-factorizable.
An inclusion-minimal set that can be removed to make it bipartite is given by
\begin{align*}
    C = \{ & (6,j), (4,n), (7,p), (18,p), (6,p), (6,n), (12,k), (10,g), (6,g), (5,p), (2,i), \\
           & (4,p), (12,m), (3,i), (12,h), (1,p), (2,q)\}.
\end{align*}
However, the formal context $(G,M,I \setminus C)$ is again not two-factorizable as its    incompatibility graph contains once again an odd cycle.
This is possible as new incompatibilities can arise with the removal of incidences.

\subsection{Computing Maximal Ordinal Two-Factorizations}

In the last sections, we did structural investigations on ordinal two-factorizations and provided an algorithm to compute them.
We now use this algorithm to compute a large ordinal two-factorization.
To this end, we propose the algorithm \textsc{Ord2Factor} in \cref{alg:ord2factor}.
Thereby, the induced bipartite subgraph of the incompatibility graph is computed.
As new incompatibilities can arise by the removal of crosses, as noted previous section, we might have to repeat this procedure.

\begin{algorithm}[t]
    \caption{\textsc{Ord2Factor} to Compute Large Ordinal Two-Factorization}
    \label{alg:ord2factor}
    \vspace*{0.3em}
    \textbf{Input:} Formal Context $(G,M,I)$\\
    \textbf{Output:} Ordinal Factors $F_1$ and $F_2$
    \vspace*{0.5em}
    \hrule
    \begin{lstlisting}
def Ord2Factor$(G,M,I)$:
    $(I, E)$ $=$ incompatibility_graph$(G,M,I)$
    while $(I,E)$ not bipartite:
        $I$ $=$ maximal_bipartite_inducing_vertex_set$(I,E)$
        $(I, E)$ $=$ incompatibility_graph$(G,M,I)$
    return two_factor$(G,M,I)$
\end{lstlisting}
\end{algorithm}

The induced bipartite subgraph can be computed using the methods proposed by Dürrschnabel et al.~\cite{bipartite}.
We are not aware of a formal context where the SAT-sovler approach described in this paper requires a second repetition of the algorithm which motivates the following open question.

\smallskip

\textsc{Open Question 1.} Is there a formal context $(G,M,I)$, such that the maximal set $\tilde{I}$ that induces a bipartite subgraph on the incompatibility graph does not give rise to a two-factorizable formal context?

\smallskip

This open problem is of special interest, because it would allow our approach to compute globally maximal two-factorization, as the following shows.

\begin{theorem}
    Let $\K=(G,M,I)$ be a formal context and $\tilde{I}$ the subset of $I$ that induces a maximal bipartite subgraph on the incompatibility graph.
    If the formal context $\K=(G,M,\tilde{I})$ admits an ordinal two-factorization, its factors are the maximal ordinal factors of $\K$.
\end{theorem}

\begin{proof}
    Assume there are two ordinal factors $F_1$ and $F_2$ of $\K$ and $|F_1 \cup F_2| > |\tilde{I}|$.
    But then the context $(G,M,F_1 \cup F_2)$ is a two-factorizable and thus the graph induced by $F_1 \cup F_2$ on the incompatibility graph is bipartite, a contradiction.
\end{proof}

\section{Runtime Discussion}
If a formal context has order dimension two, it cannot contain a contranominal of dimension three as an induced subcontext.
From a result by Albano~\cite{Albano.2015}, it follows that a context without a contranominal scale of dimension three and thus especially for all two-dimensional formal contexts, the number of concepts is bounded from above by $\frac{3}{2} |G|^2$ or dually $\frac{3}{2} |M|^2$.
There are algorithms that compute the set of all concept of a formal context with polynomial delay~\cite{Kuznetsov.2002} and the computation of a conjugate order can be performed in quadratic time~\cite{McConnell.1997}.
Thus, the algorithm to compute ordinal two-factorizations has polynomial runtime if it paired with the these algorithms.

For the computation of large ordinal factorizations of formal contexts that do not omit ordinal two-factorizations by their structure, the runtime-obstacle is the computation of the large induced bipartite subgraph.
If the exact problem is solved, the algorithm thus has exponential runtime.
Our previous work~\cite{bipartite} also  discusses three heuristics for the computation of bipartite subgraphs which can be plugged to achieve an algorithm that has polynomial runtime.
Usually, we are interested in two-factorizations of formal context with limited size, as a human can otherwise not grasp the connections encoded in the dataset.
Thus, the runtime of these algorithms are usually not the critical limitation and thus a method that is computationally expensive can be employed.

\section{Conclusion}
In this paper, we expanded on the work done on ordinal two-factorizations in the realm of ordinal factor analysis.
First, we performed some structural investigations about the disjointness of the two ordinal factors.
Thereby, we were able to characterize the incidence pairs that can appear in both ordinal factors as the isolated elements of the incompatibility graph.
Then, we proposed an algorithm for the computation of maximal ordinal two-factorizations.
To this end, we developed a polynomial time algorithm to compute a two-factorization of a formal context that has a bipartite incompatibility graph.
We showed, that the problem to compute maximal ordinal two-factorizations is $\NP$-complete and proposed our approach \textsc{Ord2Factor} to compute large ordinal two-factorizations.
As we demonstrated that the problem entails an $\NP$-complete problem, the resulting algorithm is exponential.

Datasets often consist not only of binary but also already ordinal data.
The ordinal factor analysis in its current form can only deal with this data by interpreting it as binary.
While scaling in formal concept analysis is a tool to deal with this data, factors will not necessarily respect the order encapsulated in the data.
In our opinion, the next step should be to extend this method to deal with this kind of non-binary data directly.

\bibliographystyle{template/splncs04}
\bibliography{paper}
\end{document}